\documentclass[10pt,twocolumn,letterpaper]{article}

\usepackage[pagenumbers]{cvpr} 

\usepackage{graphicx}
\usepackage{amsmath}
\usepackage{amssymb}
\usepackage{booktabs}
\usepackage{multirow}
\usepackage{color,xcolor}
\definecolor{High}{rgb}{0.98, 0.85, 0.87}
\definecolor{Light}{rgb}{0.85, 0.95, 1.0}
\usepackage{colortbl}
\usepackage{algorithm}
\usepackage[noend]{algpseudocode}
\usepackage[pagebackref,breaklinks,colorlinks]{hyperref}
\usepackage{listings}
\usepackage{xcolor}

\definecolor{codegreen}{rgb}{0,0.6,0}
\definecolor{codegray}{rgb}{0.5,0.5,0.5}
\definecolor{codepurple}{rgb}{0.58,0,0.82}
\definecolor{backcolour}{rgb}{0.95,0.95,0.92}

\lstdefinestyle{mystyle}{
    backgroundcolor=\color{backcolour},   
    commentstyle=\color{codegreen},
    keywordstyle=\color{magenta},
    numberstyle=\tiny\color{codegray},
    stringstyle=\color{codepurple},
    basicstyle=\ttfamily\footnotesize,
    breakatwhitespace=false,         
    breaklines=true,                 
    captionpos=b,                    
    keepspaces=true,                 
    numbers=left,                    
    numbersep=5pt,                  
    showspaces=false,                
    showstringspaces=false,
    showtabs=false,                  
    tabsize=2
}

\usepackage[capitalize]{cleveref}
\crefname{section}{Sec.}{Secs.}
\Crefname{section}{Section}{Sections}
\Crefname{table}{Table}{Tables}
\crefname{table}{Tab.}{Tabs.}

\usepackage{amsthm}   
\newtheorem{definition}{Definition}
\newtheorem{thm}{Theorem}

\newtheorem{lemma}{Lemma}

\newtheorem{ass}{Assumption}

\def \E {\mathrm{E}}

\def \D {\mathcal{D}}
\def \x {\mathbf{x}}

\def \X {\mathcal{X}}
\def \Y {\mathcal{Y}}
\def \F {\mathcal{F}}
\def \R {\mathbb{R}}

\def\cvprPaperID{3425} 
\def\confName{CVPR}
\def\confYear{2022}

\begin{document}

\title{Self-Supervised Pre-Training for Transformer-Based Person Re-Identification }
\newcommand*{\affaddr}[1]{#1}
\newcommand*{\affmark}[1][*]{\textsuperscript{#1}}
\author{Hao Luo, Pichao Wang, Yi Xu, Feng Ding, Yanxin Zhou, Fan Wang,  Hao Li, Rong Jin\\
\affaddr{\affmark[]Alibaba Group} \\
{\tt\small michuan.lh@alibaba-inc.com}
}
\maketitle

\begin{abstract}
Transformer-based supervised pre-training achieves great performance in person re-identification (ReID). However, due to the domain gap between ImageNet and ReID datasets, it usually needs a larger pre-training dataset (e.g. ImageNet-21K) to boost the performance because of the strong data fitting ability of the transformer. To address this challenge, this work targets to mitigate the gap between the pre-training and ReID datasets from the perspective of data and model structure, respectively. We first investigate self-supervised learning (SSL) methods with Vision Transformer (ViT) pretrained on unlabelled person images (the LUPerson dataset), and empirically find it significantly surpasses ImageNet supervised pre-training models on ReID tasks. To further reduce the domain gap and accelerate the pre-training, the Catastrophic Forgetting Score (CFS) is proposed to evaluate the gap between pre-training and fine-tuning data. Based on CFS, a subset is selected via sampling relevant data close to the down-stream ReID data and filtering irrelevant data from the pre-training dataset.  For the model structure, a ReID-specific module named IBN-based convolution stem (ICS) is proposed to bridge the domain gap by learning more invariant features. Extensive experiments have been conducted to fine-tune the pre-training models under supervised learning, unsupervised domain adaptation (UDA), and unsupervised learning (USL) settings. We successfully downscale the LUPerson dataset to 50\% with no performance degradation. Finally, we achieve state-of-the-art performance on Market-1501 and MSMT17. For example, our ViT-S/16 achieves 91.3\%/89.9\%/89.6\% mAP accuracy on Market1501 for supervised/UDA/USL ReID. Codes and models will be released to \url{https://github.com/michuanhaohao/TransReID-SSL}.
\end{abstract}
\section{Introduction}
\begin{figure}[htb]
    \centering
	\includegraphics[width=0.5\textwidth]{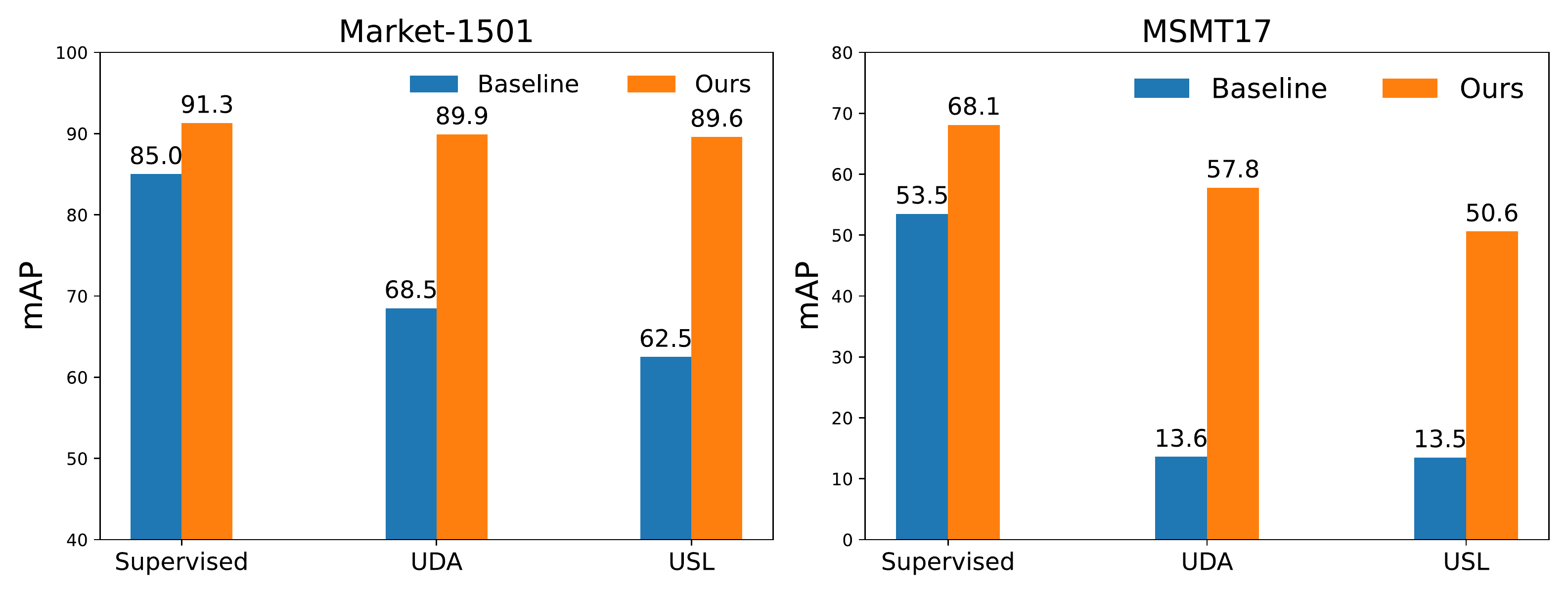}
	\caption{Performance of ViT-S on Market-1501 \cite{Market1501} and MSMT17 \cite{MSMT17} in supervised, UDA and USL ReID. Our pre-training paradigm outperforms Baseline (supervised pre-training on ImageNet + vanilla ViT) by a large margin.} 
    \vspace{-1em}
	\label{fig:intro}
\end{figure}

Transformer-based methods \cite{he2021transreid, zhang2021hat} have attracted more and more attention and achieved great performance in person ReID. For example, the pure transformer-based method TransReID \cite{he2021transreid} achieves a significant performance boost over state-of-the-art CNN-based methods. However, there exists a large domain gap between ImageNet and person ReID datasets because 1) The image content of ImageNet and ReID datasets is very different \cite{fu2021unsupervised}; 2) supervised pre-training on ImageNet is focused on category-level supervision which reduces the rich visual information \cite{dino} while person ReID prefers fine-grained identity information. As a result, Transformers need to be pre-trained on a larger-scale dataset ImageNet-21K~\cite{deng2009imagenet} to avoid over-fitting on the pre-training dataset. To bridge the gap between pre-training and fine-tuning datasets for better transformer-based ReID models, this paper tackles the problem from the perspectives of data and model structure, respectively. 

From the data view, we have seen a large-scale pre-training dataset named LUPerson being built by collecting unlabeled person images~\cite{fu2021unsupervised}, which has demonstrated that CNN-based SSL pre-training on LUPerson dataset improves ReID performance compared with ImageNet-1k pre-traininng. However, following the same paradigm and replacing the backbone with Vision Transformers (ViTs)~\cite{VIT} would get poor performance due to the huge differences between training CNNs and ViTs. It motivates us to explore an effective SSL pre-training paradigm for transformer-based person ReID in the first place. After thorough comparison between several Transformer-based self-supervised methods (\eg MocoV3 \cite{chen2021mocov3}, MoBY \cite{moby} and DINO \cite{dino}) with ViTs on the LUPerson dataset, it is found that DINO significantly outperforms other SSL methods and supervised pre-training on ImageNet, thus it is used as our following setup. Next, it is brought to our attention that, though with better performance, pre-training on LUPersonn needs a larger amount of computational resources due to the large amount of training data (3X of ImageNet-1K). Therefore, we propose to adopt \textit{conditional pre-training} to speed-up the training process and further reduce the domain gap. As LUPerson is collected from web videos, a portion of the images are of low-quality or have great domain bias with the downstream ReID datasets, so \textit{conditional filtering} shall be performed to downscale LUPerson (source domain) by selecting a relevant subset close to the downstream ReID datasets (target domain). However, previous works on conditional filtering ~\cite{chakraborty2020efficient, coleman2020selection, ge2017borrowing, cui2018large, Fine} are mainly designed for close-set tasks by selecting data close to category labels or cluster centers of the target training data. Those methods will easily overfit the person IDs instead of the true target domain if directly applied to the open-set ReID task. We propose a metric named Catastrophic Forgetting Score (CFS) to evaluate the gap between pre-training and downstream data, inspired by the conclusions in catastrophic forgetting problems~\cite{thompson2019overcoming,kirkpatrick2017overcoming,ramasesh2021anatomy}. For a pre-training image, the similarity between its features from the two proxy models (one is pre-trained on the source dataset and another is fine-tuned on the target dataset) can represent the similarity between the source and target domain. In this way, a subset of images with higher CFS scores can be selected from the pre-training data to perform efficient conditional pre-training. A preliminary theoretical analysis is conducted to justify the effectiveness of CFS.

From the perspective of model structure, some recent works \cite{chen2021mocov3, conv_stem,wang2021scaled} have pointed out that, an important factor that affects performance and stability of ViTs is the \textit{patchify stem} implemented by a stride-$p$ $p \times p$ convolution ($p = 16$ by default) on the input image. To address this problem, MocoV3 \cite{chen2021mocov3} froze the patch projection to train ViTs, while Xiao \etal~\cite{conv_stem} and Wang \etal~\cite{wang2021scaled}  proposed a \textit{convolution stem} stacked by several convolution, Batch Normalization (BN)~\cite{bn}, and ReLU~\cite{relu} layers to increase optimization stability and improve  performance. Inspired by the success of integrating Instance Normalization (IN) and BN to learn domain-invariant representation in the ReID task~\cite{BoT, ibnnet, dai2021cluster}, we refer to IBN-Net \cite{ibnnet} and improve the convolution stem to the \textit{IBN-based convolution stem (ICS)}. ICS inherits the stability of the convolution stem and also introduces IBN to learn features with appearance invariance (\eg viewpoint, pose and illumination invariance, etc.). ICS has similar computational complexity to convolution stem, but improves peak performance of ViTs significantly in person ReID.

Embraced with the above two improvements based on DINO, we conduct experiments with supervised learning, UDA and USL settings on Market1501 and MSMT17. Our pre-training paradigm helps ViTs achieve state-of-the-art performance on these benchmarks. For instance, as shown in Figure~\ref{fig:intro}, our ViT-S achieves 91.3\%/89.9\%/89.6\% mAP accuracy on Market-1501 for supervised/UDA/USL ReID, which exceed the supervised pre-training setting on ImageNet by a large margin. The pre-training cost on LUPerson is also reduced by 30\% without any performance degradation, through the proposed CFS-based conditional filtering and downscaling the original LUPerson by 50\%. 



\section{Related Works}

\subsection{Self-supervised Learning}
Self-supervised learning (SSL) methods are proposed to learn discriminative features from large-scale unlabeled data without any human-annotated labels \cite{jing2020self}. 
One branch is developed from Momentum Contrast (MoCo)~\cite{moco_v2} by treating a pair of augmentations of a sample as a positive pair and all other samples as negative pairs. Since the number of negative samples greatly affects the final performance, MoCo series~\cite{chen2021mocov3, moco_v2} require large batches or memory banks. Among them, MoCoV3~\cite{chen2021mocov3} is a Transformer-specific version. Fu \etal~\cite{fu2021unsupervised} have verified ResNet50 can be well pre-trained by a modified MoCoV2 on human images in person ReID. Many recent works have shown models can learn feature representation without discriminating between images. In this direction, Ge \etal~\cite{BYOL} propose a new paradigm called BYOL, where the online network predicts the representation of the target network on the same image under a different augmented view. Large batch size is unnecessary in BYOL since negative samples are not needed. Many variants have successfully improved BYOL in various ways. One of them is  DINO~\cite{dino}, where a centering and sharpening of the momentum teacher outputs is used to avoid model collapse. DINO achieves state-of-the-art performance with ViTs on both ImageNet classification and down-stream tasks. Xie \etal combines MoCo with BOYL to propose a Transformer-specific method called MoBY. Given that there have been more SSL methods designed specifically for Transformers, we will focus on several of the state-of-the-art options in the following experiments, \eg MoCo series, MoBY and DINO.

\subsection{Transformer-based ReID}

CNN-based methods have dominated the ReID community for many years. However, pure-transformer models are gradually becoming a popular choice. He \etal~\cite{he2021transreid} are the first to successfully apply ViTs to ReID tasks by proposing TransReID which achieves state-of-the-art performance on both person and vehicle ReID. Auto-Aligned Transformer (AAformer)~\cite{zhu2021aaformer} also uses ViT backbone with the additional learnable vectors of ``part tokens" to learn the part representations and integrates the part alignment into the self-attention. Other works try to use Transformer to aggregate features or information from CNN backbones. For example, \cite{zhang2021hat,shen2021git,PAT} integrate Transformer layers into the CNN backbone to aggregate hierarchical features and align local features. For video ReID, \cite{zhang2021spatiotemporal, liu2021video} exploit Transformer to aggregate appearance features, spatial features, and temporal features to learn a discriminative representation for a person tracklet. 

\subsection{Conditional Transfer Learning}

There are a few works~\cite{chakraborty2020efficient, coleman2020selection, ge2017borrowing, cui2018large, Fine, yan2020neural} studying how to select relevant subsets from the pre-training dataset to improve the performance when transferred to target datasets. \cite{cui2018large} use a feature extractor trained on the JFT300M~\cite{hinton2015distilling} to select the most similar source categories to the target categories in a greedy way. For each image from the target domain, Ge \etal~\cite{ge2017borrowing} search a certain number of images with similar low-level characteristics from the source domain. Shuvam \etal~\cite{chakraborty2020efficient} train the feature extractor on the target data and individually select source images that are close to cluster centers in the target domain. Yan \etal~\cite{yan2020neural} propose a Neural Data Server (NDS) to train expert models on many subsets of the pre-training dataset. Source images used to train an expert with a good target task performance are assigned high importance scores. It is noted that conditional transfer learning has mainly been studied on close-set tasks such as image classification, fine-grained recognition, object detection, etc. Therefore, these methods may be not suitable for the open-set ReID task.

\section{Self-supervised Pre-training}

As far as we know, there has been no literature studying the SSL pre-training for transformer-based ReID. Therefore we first conduct an empirical study to gain better understanding on this problem. We investigate two backbones (CNN-based ResNet50 and Transformer-based ViT), four SSL methods (MoCoV2, MoCoV3, MoBY and DINO), and two pre-training datasets (ImageNet and LUPerson). MoCoV2 used here is a modified version proposed to adapt to person ReID on ResNet50 in \cite{fu2021unsupervised}, while the other three methods, \ie MoCoV3, MoBy and DINO, are transformer-specific methods proposed on the ImageNet data. 

\subsection{Supervised Fine-tuning}

\renewcommand{\multirowsetup}{\centering}
\begin{table}[tb]\small
    \begin{center}
    \setlength\tabcolsep{5pt}
    \begin{tabular}{ccc|cc|cc}
    \hline
     \multicolumn{3}{c|}{Pre-training}  & \multicolumn{2}{c|}{Market}  & \multicolumn{2}{c}{MSMT17}\\
    Models & Methods &Data & mAP & R1 & mAP & R1\\
    \hline
    \multirow{4}{*}{R50}& Supervised &IMG &86.7 &94.8 &52.2 &76.0\\
    & MoCoV2 &LUP &88.2 &94.8 &53.3 &76.0\\
    & MoCoV3 &LUP &87.3 &95.1 &52.9 &76.8\\
    & DINO &LUP &86.5 &94.4 &51.9 &75.8\\
    \hline
    \multirow{5}{*}{ViT-S/16}& Supervised &IMG &85.0 &93.8 &53.5 &75.2\\
    &MoCoV2 &IMG &63.6 &72.1 &19.6 &36.1\\
    &MoCoV3 &IMG &81.7 &92.1 &46.6 &70.3\\
    &MoBY &IMG &83.3 &92.2 &49.1 &71.5\\
    &DINO &IMG &84.6 &93.1 &54.8 &76.7\\
    \hline
    \multirow{6}{*}{ViT-S/16}&MoCoV2 &LUP &72.1 &87.6 &27.8 &47.4\\
    &MoCoV3 &LUP &82.2 &92.1 &47.4 &70.3\\
    &MoBY &LUP &84.0 &92.9 &50.0 &73.2\\
    &DINO &LUP &\textbf{90.3} &\textbf{95.4} &\textbf{64.2} &\textbf{83.4}\\
    &DINO &LUP$^{*}$ &89.6 &95.1 &62.3 &82.6\\
    \hline
    \end{tabular}
    \end{center}
    \vspace{-1em}
    \caption{\label{tab:ssl}Comparison of different pre-training models. We pre-trained ResNet50 (R50) and ViT-S/16 on ImageNet-1K (IMG) and LUPerson (LUP) datasets. MoBY doesn't provide training settings for ResNet50. To fairly compare with ImageNet, we randomly sampled 1.28M images from LUPerson to build a subset denoted as LUP$^*$.}
    \vspace{-1em}
\end{table}
The baseline \cite{he2021transreid} we used is pre-trained on ImageNet. We fine-tune pre-trained models on two ReID benchmarks including Market-1501 (Market) and MSMT17, and the peak fine-tuning performance of all models being compared are presented in Table \ref{tab:ssl}. For convenience, a pre-trained model is marked in the form of \textit{Model+Method+Data}. For instance, R50+Supervised+IMG denotes the ResNet50 model pre-trained on the ImageNet in the supervised manner, which is the standard pre-training paradigm in most previous ReID methods. 

Some observations are made as follows. 1) MoCoV2 performs the best among all SSL methods for ResNet50, while it is much worse than the other three transformer-specific methods with ViT, which means that it is necessary to explore specific methods for transformer-based models. 2) ViT is more sensitive to a proper pre-training than ResNet50. For example, we can see that mAP using ResNet50 ranges from 51.9\% to 53.3\% on MSMT17 with different pre-training settings, while performance of ViT-S/16 differs more widely on MSMT17. 3) SSL methods on LUP consistently achieve better performance comparing to SSL with ImageNet. Even when we restrict the number of training images of LUP to be the same as IMG,  ViT-S/16+DINO+LUP$^*$ still surpasses ViT-S/16+DINO+
IMG on both benchmarks, indicating that leveraging person images is a better choice to pre-train ReID models. 4) ViT-S/16+DINO+LUP achieves 64.2\% mAP and 83.4\% Rank-1 accuracy on the MSMT17, surpassing the baseline (ViT-S/16+Supervised+IMG) by 10.7\% mAP and 8.2\% Rank-1 accuracy.

\subsection{Unsupervised Fine-tuning}

\renewcommand{\multirowsetup}{\centering}
\begin{table}[htb]\footnotesize
    \begin{center}
    \setlength\tabcolsep{2pt}
    \begin{tabular}{ccc|cc|cc}
    \hline
     \multicolumn{3}{c|}{Pre-training}  & \multicolumn{2}{c|}{Market}  & \multicolumn{2}{c}{MSMT17}\\
    Models & Methods &Data & mAP & R1 & mAP & R1\\
    \hline
    \multirow{2}{*}{R50}& Sup &IMG &82.6 &93.0 &33.1 &63.3 \\
    & SSL &LUP &84.0 &93.4 &31.4 &58.8\\
    \hline
    \multirow{3}{*}{ViT-S/16}& Sup &IMG &62.5 &80.5 &13.5 &29.9 \\
    & SSL &IMG  &68.9{\color{gray}(+6.4)} &84.3{\color{gray}(+3.8)} &14.9{\color{gray}(+1.4)} &31.0{\color{gray}(+1.1)}\\
    & SSL &LUP &87.8{\color{gray}(+25.3)} &94.4{\color{gray}(+13.9)} &38.4{\color{gray}(+24.9)} &63.8{\color{gray}(+33.9)}\\
    \hline
    \end{tabular}
    \end{center}
    \vspace{-1em}
    \caption{\label{tab:usl} Model performance in USL ReID. Supervised pre-training and self-supervised pre-training are abbreviated as SUP and SSL, respectively. }
    \vspace{-1em}
\end{table}

\renewcommand{\multirowsetup}{\centering}
\begin{table}[htb]\footnotesize
    \begin{center}
    \setlength\tabcolsep{2pt}
    \begin{tabular}{ccc|cc|cc}
    \hline
     \multicolumn{3}{c|}{Pre-training}  & \multicolumn{2}{c|}{MS2MA}  & \multicolumn{2}{c}{MA2MS}\\
    Models & Methods &Data & mAP & R1 & mAP & R1\\
    \hline
    \multirow{2}{*}{R50}& Sup &IMG & 82.4 & 92.5 & 33.4 & 60.5\\
    & SSL &LUP  & 85.1 & 94.4 &28.3 &53.8\\
    \hline
    \multirow{3}{*}{ViT-S/16}& Sup &IMG & 68.5 & 85.4 & 13.6 & 29.5   \\
    & SSL &IMG &79.7{\color{gray}(+11.2)} &90.5{\color{gray}(+5.1)} &21.8{\color{gray}(+8.2)} &41.6{\color{gray}(+12.1)} \\
    & SSL &LUP &88.5{\color{gray}(+20.0)} &95.0{\color{gray}(+9.6)} &43.9{\color{gray}(+30.3)} &67.7{\color{gray}(+38.2)} \\
    \hline
    \end{tabular}
    \end{center}
    \vspace{-1em}
    \caption{\label{tab:uda} Model performance in UDA ReID. MS2MA and MA2MS stands for MSMT17$\rightarrow$Market (MS2MA) and Market$\rightarrow$MSMT17 (MA2MS), respectively.}
    \vspace{-1em}
\end{table}

Since there has been no transformer-based baseline for unsupervised ReID, the state-of-the-art CNN-based framework C-Contrast \cite{dai2021cluster} is selected for the following experiments. We reproduce C-Contrast with ResNet50 and ViT-S/16 on both USL ReID and UDA ReID \footnote{\url{https://github.com/alibaba/cluster-contrast-reid}. All results are reproduced using the official code with the same random seed and GPU machine to reduce randomness.}. Based on observations made from Table~\ref{tab:ssl}, we choose MoCoV2 for ResNet50 and DINO for ViT-S/16 in this section in Table~\ref{tab:usl}. SSL pre-training doesn't provide large gains compared with Sup when applied with ResNet50, not to mention that performance drop is observed on MSMT17. This is consistent with the aforementioned conclusion that Transformer is more sensitive to the pre-training than CNN. In contrast, SSL pre-training on LUPerson improves the performance by a large margin for ViT-S/16. The UDA performance on MS2MA and MA2MS also yields the similar conclusion. 

\noindent\textbf{Conclusion.} DINO is the most suitable SSL method among candidate methods for transformer-based ReID. Pre-training is more crucial for Transformers than CNN models.  Transformers pre-trained on LUPerson can significantly improve the performance compared with ImageNet pre-training, indicating that bridging the domain gap between pre-training and fine-tuning datasets for transformers is more beneficial and worth doing. 

\section{Conditional Pre-training}
This section introduces the efficient conditional pre-training where models are pre-trained on a subset closer to the target domain to speed up the pre-training process while maintaining or even improving the downstream fine-tuning performance. Catastrophic Forgetting Score (CFS) is proposed to evaluate the similarity between the pre-training data and the target domain. Theoretic analysis is provided to support the method. 

\subsection{Problem Definition}

Given a target dataset $\mathcal{D}_t = (\mathcal{X}_t,\mathcal{Y}_t)$ where $\mathcal{X}_t = \{x_t^1, x_t^2, x_t^3,...,x_t^M \}$ with their ID labels $\mathcal{Y}_t$. We target to select a subset $\mathcal{D}_{s}^{\prime}$ from a large-scale source/pre-training dataset $\mathcal{D}_{s}$ where $\mathcal{X}_s = \{x_s^1, x_s^2, x_s^3,...,x_s^N \}$. 
The number of images in $\mathcal{D}_{s}^{\prime}$ is $N^{\prime} < N$. The efficient conditional pre-training is to pre-train models on $\mathcal{D}_{s}^{\prime}$, which should reduce the training cost of pre-training while maintaining or even improving performance on the target dataset. Some previous works \cite{chakraborty2020efficient, coleman2020selection, Fine} have shown that the solution is to select pre-training data close to the target domain, and we also provide theoretical analysis in Appendix to further verify this. Since ReID is a open set problem with different IDs in training and testing set, the key problem is how to design a metric to evaluate the `similarity' between the pre-training data $x_s^i, i \in [1,N]$ and the target dataset $\mathcal{D}_{t}$ (instead of the person IDs in $\mathcal{D}_{t}$).

\subsection{Catastrophic Forgetting Score}

\vspace{-1em}
\begin{algorithm}[!ht]
\caption{Our Proposed Conditional Filtering}\label{alg:clustering}
\begin{algorithmic}[1]
\Procedure{Filter}{$\mathcal{D}_{s},\mathcal{D}_{t}$}
\State $\theta_s \gets TRAIN(\mathcal{D}_{s})$ \Comment{Source Proxy Model}
\State $\theta_t \gets TRAIN(\theta_s, \mathcal{D}_{t})$ \Comment{Target Proxy Model}
\For{$i \gets 1$ to $N$} \Comment{$x_s^i \in \mathcal{D}_{s}$}
\State $c_s^i \gets CFS(x_s^i)$ \Comment{Compute CFS}
\EndFor
\State $c_s \gets SORT(c_s^1, c_s^2,...,c_s^N)$ \Comment{Get Score Set}
\State $\mathcal{D}_{s}^{\prime} \gets TOP(\mathcal{D}_{s}, N^{\prime},c_s)$ \Comment{Filter Source Dataset}
\State \textbf{return} $\mathcal{D}_{s}^{'}$\Comment{Return the Filtered Subset}
\EndProcedure
\end{algorithmic}
\end{algorithm}

We first pre-train a model $\theta_s$ on the source dataset $\mathcal{D}_{s}$. $\theta_s$ is transferred into $\theta_t$ by fine-tuning on the target dataset $\mathcal{D}_{t}$. Since $\theta_s$ and $\theta_t$ only serve as proxy models to select data, there is no need for them to achieve the best ReID performance, \ie they can be lightweight models trained with less epochs. In this paper, $\mathcal{D}_{t}$ is the fusion of Market-1501 and MSMT17, so we only need to select one subset $\mathcal{D}_{s}^{\prime}$ sharing between different ReID datasets.

Many previous works \cite{thompson2019overcoming,kirkpatrick2017overcoming,ramasesh2021anatomy} have observed that catastrophic forgetting in neural networks occurs during domain transferring, and the degree of forgetting is related to the gap between the source domain and the target domain. To evaluate the domain gap between the pre-training data and the target domain, a simple metric named Catastrophic Forgetting Score (CFS) is proposed as below, which computes the representation similarity between $\theta_s(x_s^i)$ and $\theta_t(x_s^i)$ for the pre-training data $x_s^i$:
\begin{align}\label{CFS}
c_s^i = \frac{\langle\theta_s(x_s^i), \theta_t(x_s^i)\rangle}
{||\theta_s(x_s^i)||\, ||\theta_t(x_s^i)||}.
\end{align}
The greater the $c_s^i$, the smaller the degree of forgetting, \ie the smaller the domain gap.  $c_s = \{c_s^1, c_s^2,...,c_s^N\}$ is sorted in descending order and the top $N^{\prime}$ images are selected to get the subset $\mathcal{D}_{s}^{\prime} \in \mathcal{D}_{s}$.

The advantage of  CFS is that it does not compare $x_s^i$ directly with the images in the target domain. That is, unlike previous methods designed for close-set tasks, it avoids scanning through all images from the target domain to compute $c_s^i$, which is computationally costly. 

\subsection{Theoretical Analysis of CFS}\label{subsec:math:cfs}
In this subsection, we give a theoretical analysis to show large CFS is a necessary condition for finding an image in the source domain that is close to the target domain. To this end, we first rigorously define some terminologies as follows, which will be used in the analysis.
\begin{definition}
The feature representation $\theta$ is a function of $x\in\X$ that maps to $\R^d$, i.e., $\theta(x): \X\rightarrow\R^d$.
\end{definition}
The feature representation $\theta$ is a function extracting features of an image. Without loss of generality, we use the normalized representation function in our analysis, i.e., 
\begin{align}\label{def:tilde:theta}
    \tilde\theta := \frac{\theta}{\|\theta\|}, 
\end{align}
where $\|\cdot\|$ is the norm.
\begin{definition}
We say $x$ and $x'$ are $(
\epsilon,\theta_1,\theta_2)$-similar, if $\|\tilde\theta_1(x)-\tilde\theta_2(x')\|\le \epsilon$, where $\epsilon\in(0,1)$ is a small constant, $\theta_1$ and $\theta_2$ are two representation functions, $\tilde\theta_1$ and $\tilde\theta_2$ are defined as (\ref{def:tilde:theta}). In addition, if $\theta_1 = \theta_2 := \theta$, we use $(
\epsilon,\theta)$-similar for simplicity.
\end{definition}
The similarity of two images can be defined as the ``closeness" of their features. This is reasonable since the feature is considered as the representation of an image.
\begin{ass}\label{ass:similarity}
If $x\in\D$ and $x'\in\D'$ are close, then they are $(\epsilon/2,\theta,\theta')$-similar, i.e.,
   $\|\tilde\theta(x) - \tilde\theta'(x')\|\le\epsilon/2$,
where $\theta:\D\rightarrow\R^d$ and $\theta':\D'\rightarrow\R^d$ are two representation functions, $\tilde\theta$ and $\tilde\theta'$ are defined as (\ref{def:tilde:theta}).
\end{ass}
Since our inference focuses on downstream tasks, it is interested in considering the representation for the target domain (i.e., $\theta_t$). The main goal is to find an image in the source domain that is $(\epsilon, \theta_t)$-similar to an image from the target domain. Mathematically, we aim to find $x_s^i\in\D_s$ to be $(
\epsilon,\tilde\theta_t)$-similar to $x_t^j\in\D_t$\footnote{Please note that $x_t^j$ could not be a true image from target domain and it is for the proof use only. In general, the image $x_t^j$ can be a virtual one that follows the same distribution of the images of target domain.}, that is,
\begin{align}\label{eqn:dis:1}
    \|\tilde\theta_t(x_s^i) -  \tilde\theta_t(x_t^j)\|\le \epsilon.
\end{align}
The following theorem shows that the large CFS is a necessary condition for finding such an image in the source domain.
\begin{thm}\label{main:CFS}
For a given $x_s^i\in\D_s$, if there exists a $x_t^j\in\D_t$ that satisfies Assumption~\ref{ass:similarity}, then we have the following result. If
\begin{align}\label{CFS:eqn}
    c_s^i \ge 1-\epsilon^2/8,
\end{align}
then (\ref{eqn:dis:1}) holds, where $c_s^i$ is defined in (\ref{CFS}) and $\epsilon\in(0,1)$ is a small constant.
\end{thm}

\begin{proof}\renewcommand{\qedsymbol}{}
Let define the distance 
\begin{align}\label{eqn:dis:condition}
    d_{s,t} := \|\tilde\theta_t(x_s^i) - \tilde\theta_s(x_s^i)\|.
\end{align}
We write the distance $\|\tilde\theta_t(x_s^i) -  \tilde\theta_t(x_t^j)\|$ as
\begin{align}\label{eqn:dis:2}
    \nonumber &\|\tilde\theta_t(x_s^i) -  \tilde\theta_t(x_t^j)\| \\
    \nonumber = &
    \|\tilde\theta_t(x_s^i) - \tilde\theta_s(x_s^i) + \tilde\theta_s(x_s^i) - \tilde\theta_t(x_t^j)\| \\
    \le & 
    d_{s,t} + \|\tilde\theta_s(x_s^i) - \tilde\theta_t(x_t^j)\|,
\end{align}
where the last inequality uses Triangle Inequality for norm. 
Furthermore, if $x_s^i$ and $x_t^j$ are close, then under Assumption~\ref{ass:similarity} we have
\begin{align}\label{eqn:dis:3}
    \|\tilde\theta_s(x_s^i) - \tilde\theta_t(x_t^j)\| \le \epsilon/2.
\end{align}
Next, we want to explain why the small distance $ d_{s,t}$ can be used as the selection criterion, let consider two cases. When $d_{s,t} \ge \|\tilde\theta_s(x_s^i) - \tilde\theta_t(x_t^j)\|$, by (\ref{eqn:dis:condition}) and (\ref{eqn:dis:2}) we have
\begin{align} \label{eqn:dis:condition:1}
    \|\tilde\theta_t(x_s^i) -  \tilde\theta_t(x_t^j)\| \le 2d_{s,t}.
\end{align}
When $d_{s,t} \le \|\tilde\theta_s(x_s^i) - \tilde\theta_t(x_t^j)\|$, by (\ref{eqn:dis:2}) and (\ref{eqn:dis:3}), we have
\begin{align} \label{eqn:dis:condition:2}
   \|\tilde\theta_t(x_s^i) -  \tilde\theta_t(x_t^j)\| \le \epsilon
\end{align}
Thus, if the distance $d_{s,t}\le\epsilon/2$, then by (\ref{eqn:dis:condition:1}) and (\ref{eqn:dis:condition:2}) we know $x_t^j$ and $x_s^i$ are $(\epsilon,\tilde\theta_t)$-similar, i.e., (\ref{eqn:dis:1}) holds.

The above condition of $d_{s,t}\le\epsilon/2$ is equivalent to (\ref{CFS:eqn}) 
since $d_{s,t}^2 = \|\tilde\theta_t(x_s^i) - \tilde\theta_s(x_s^i)\|^2 = \|\tilde\theta_t(x_s^i) \|^2 + \| \tilde\theta_s(x_s^i)\| - 2 \langle \tilde\theta_s(x_s^i), \tilde\theta_t(x_s^i)\rangle = 2 - 2 \langle \tilde\theta_t(x_s^i), \tilde\theta_s(x_s^i)\rangle = 2-2c_s^i$. So the greater $c_s^i$ is, the smaller $d_{s,t}$ is.
\end{proof}

Although larger value of CFS is proved to be only a necessary condition for finding a source domain image close to the downstream task, Table~\ref{tab:cfs} also empirically shows that this metric is effective in practice.

\section{IBN-based Convolution Stem}
\begin{figure}[htb]
    \centering
    \vspace{-1em}
	\includegraphics[width=0.48\textwidth]{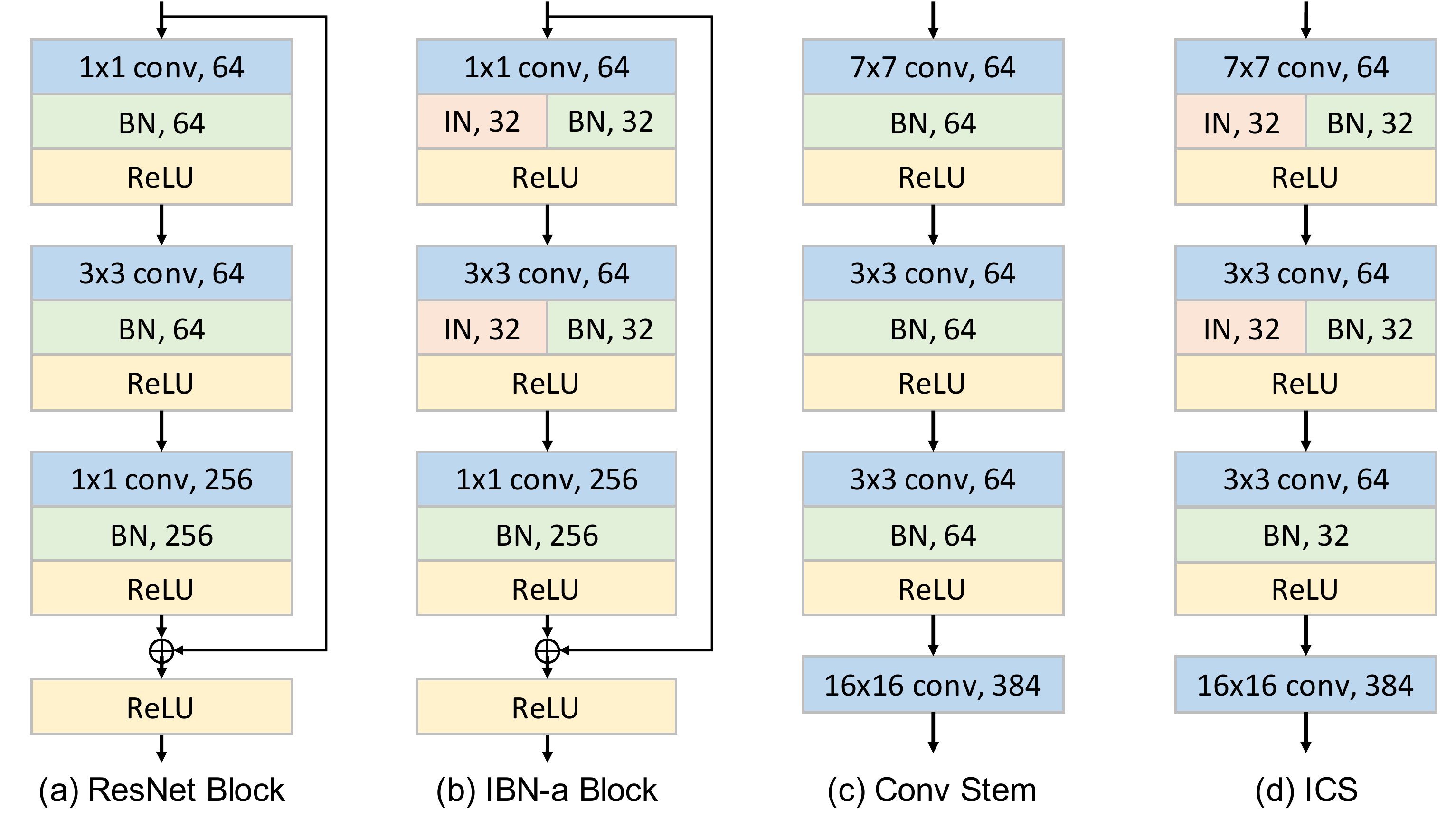}
	\caption{Comparison of different modules. We refer to the IBNNet-a block to design the IBN-based Convolution Stem (ICS). (a) ResNet block, (b) IBNNet-a block \cite{ibnnet}, (c) vanilla convolution stem \cite{conv_stem}, (d) our proposed ICS.} 
	\label{fig:ibn}
\end{figure}

Our architecture design is based on two important experiences. 1) The \textit{patchity stem} implemented by a stride-$p$ $p \times p$ convolution ($p = 16$ by default) in the standard ViT is the key reason of training instability~\cite{chen2021mocov3}. Recent works show the \textit{convolution stem}~\cite{conv_stem,wang2021scaled} stacked with several convolution, BN and ReLU layers can significantly improve training stability and peak performance. 2) Data bias is a critical challenge for person ReID. By learning invariant representation, IBNNet-a block~\cite{ibnnet} has achieved a huge success in many publications or academic challenges \cite{luo2019strong, dai2021cluster}.
    
When applying ViTs in ReID tasks, a straightforward way is to introduce IBNNet-a block into the convolution stem. Following the original design of IBNNet-a, we propose \textit{IBN-based convolution stem (ICS)}, as shown in Figure~\ref{fig:ibn}, by applying BN for half of the channels and IN for the other half after shallow convolution layers, and applying only BN layers after deep convolution layers. In this paper, we choose to apply IN after the first two convolution layers. Another variant worth consideration but may slightly reduce performance is to only apply IN after the first convolution layer. The kernel sizes, channels and strides of the convolution stem are kept the same as in \cite{conv_stem}. Hereinafter, we denote ViT with patchify stem, convolution stem, and ICS as ViT, ViT$_C$ and ViT$_I$, respectively

\section{Experiments}

\subsection{Implementation Details}

\noindent\textbf{Datasets.} The dataset used for pre-training, LUPerson~\cite{fu2021unsupervised}, contains 4.18M unlabeled human images in total, collected from 50,534 online videos. To evaluate the performance on ReID tasks, we conduct experiments on two popular benchmarks, \ie Market-1501 (Market) \cite{Market1501} and MSMT17 \cite{MSMT17}. They contain 32,668 images of 1,501 identities and 126,441 images of 4,101 identities, respectively. Images in these two datasets are resized to $256 \times 128$ during training and inference stages. Standard evaluation protocols are used with the metrics of mean Average Precision (mAP) and Rank-1 accuracy.

\noindent\textbf{Pre-training.} Unless otherwise specified, we follow the default training settings of all self-supervised methods. Images are resized to $224 \times 224$ and $256 \times 128$ for ImageNet and LUPerson, respectively. For pre-training on LUPerson with DINO, the model is trained on $8 \times$V100 GPUs for 100 epochs and a multi-crop strategy is applied to crop 8 images with $96 \times 48$ resolution. 

\noindent\textbf{Supervised ReID.} The transformer-based baseline proposed in \cite{he2021transreid} is used as our baseline in this paper. That is to say, none of the overlapping patch embedding, jigsaw patch module or side information embedding is included here. It is noticed that DINO models should be fine-tuned with a small learning rate and a longer warm-up process, so we recommend setting the learning rate to $lr = 0.0004 \times \frac{batchsize}{64}$ and warm-up epochs to be 20. All other settings are same with the original paper.

\noindent\textbf{USL/UDA ReID.} We follow most of the default settings in \cite{dai2021cluster}. USL ReID and UDA ReID share the same training settings in this paper, and the only difference is that the models for UDA ReID need to be first pre-trained on source datasets before training in an unsupervised manner on target datasets. The maximum distance $d$ between two samples is set to 0.6 and 0.7 for Market1501 and MSMT17, respectively. We use SGD optimizer to train ViTs for 50 epochs. The initial learning rate is set to 3.5e-4, and is reduced 10 times for every 20 epochs. Each mini-batch contains 256 images of 32 person IDs, \ie each ID contains 8 images. The rate of stochastic depth is set 0.3.

\subsection{Results of Conditional Pre-training}

\vspace{-1em}
\begin{figure}[htb]
    \begin{minipage}[t]{0.48\linewidth}
    \centering
    \includegraphics[width=1.05\textwidth]{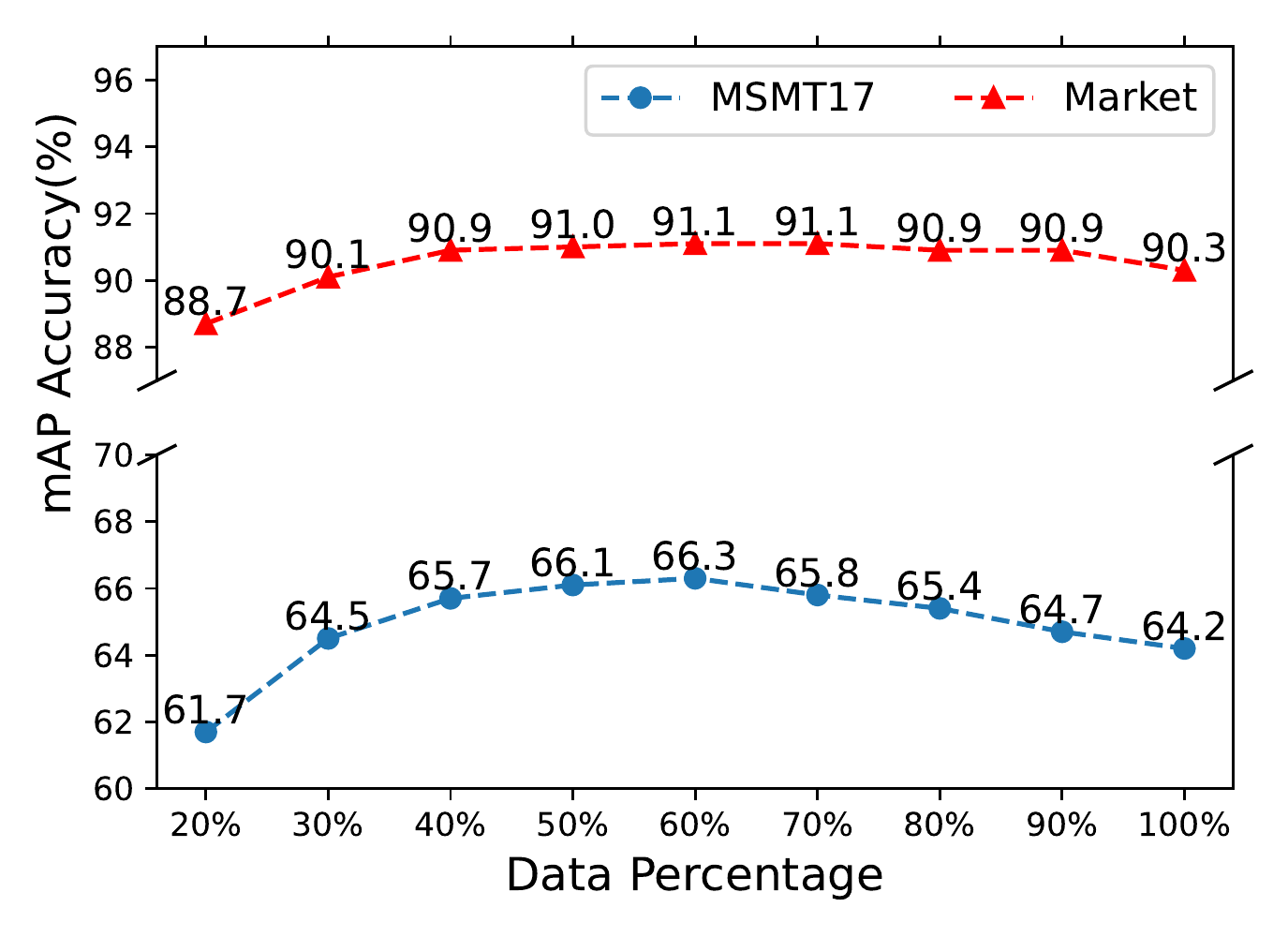}
    \end{minipage}
    \begin{minipage}[t]{0.49\linewidth}
    \centering
    \includegraphics[width=1.05\textwidth]{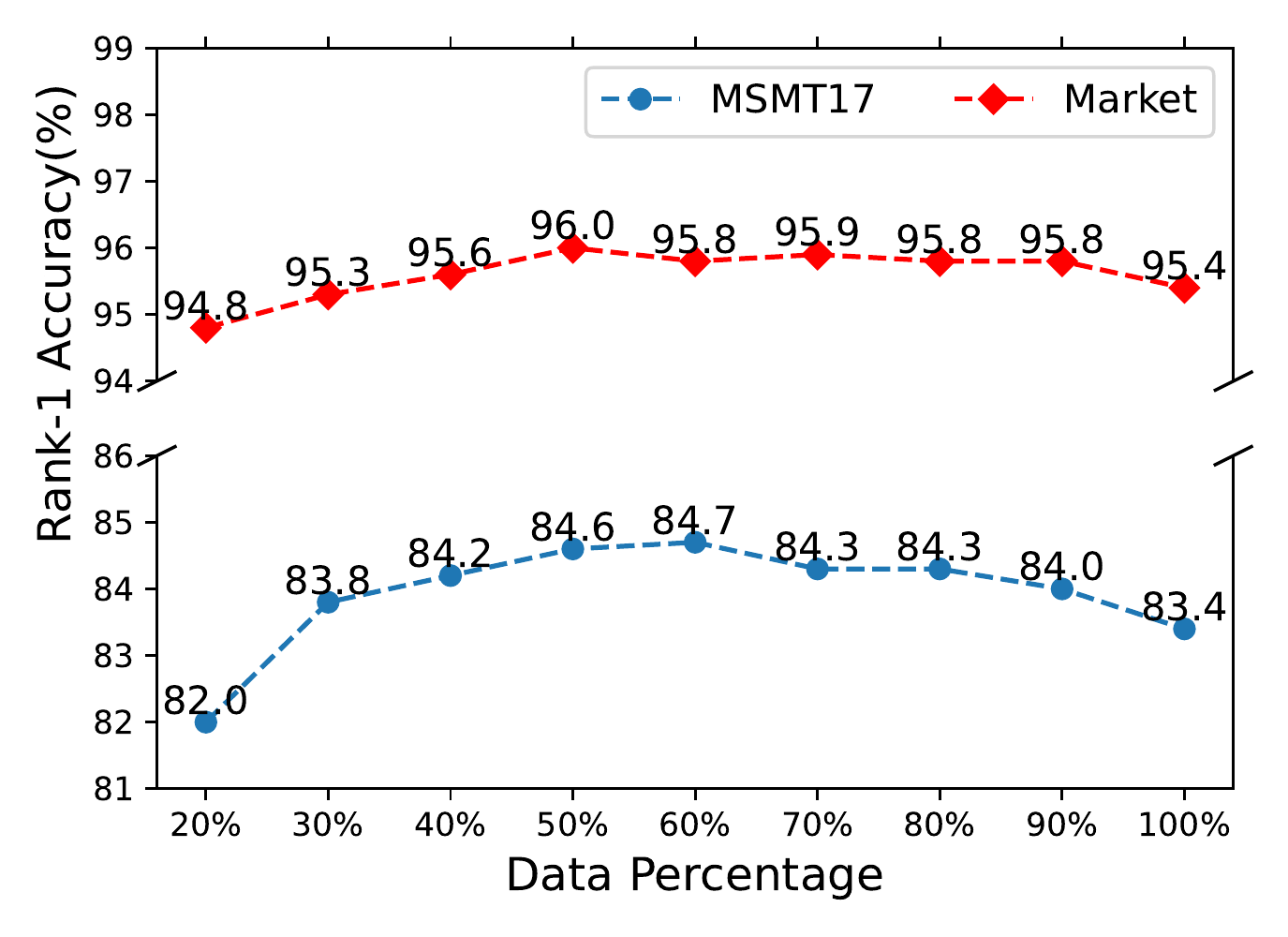}
    \end{minipage}
	\caption{Supervised fine-tuning performance of our conditional pre-training with different percentages of pre-training data. All results are achieved by ViT-S. Left and right figures show mAP and Rank-1 accuracy, respectively.} 
	\label{fig:cfs}
\end{figure}

\begin{table*}
\begin{minipage}[t]{0.48\linewidth}
    \begin{center}
    \resizebox{0.9\linewidth}{!}{
        \begin{tabular}{c|cc|cc}
        \hline
        \multirow{2}{*}{Pre-training} & \multicolumn{2}{c|}{Market}  & \multicolumn{2}{c}{MSMT17}\\
         & mAP & R1  & mAP & R1\\
        \hline
        Full (100\%)    &90.3 &95.4 &64.2 &83.4 \\
        Random (50\%)   &89.9 &95.2 &62.9 &82.6 \\
        Cluster (50\%) \cite{chakraborty2020efficient} &90.0 &95.4 &63.8 &83.3 \\
        CFS (50\%) &\textbf{91.0} &\textbf{96.0} &\textbf{66.1} &\textbf{84.6} \\
        \hline
        \end{tabular}
        }
    \end{center}
    \vspace{-1em}
    \caption{\label{tab:cfs} Comparison of different data selection strategies. Random sampling (Random) and selecting data close to cluster centers (Cluster) are compared.}
\end{minipage}
\hfill
\begin{minipage}[t]{0.48\linewidth}
    \begin{center}
    \resizebox{0.9\linewidth}{!}{
        \begin{tabular}{c|cc|cc}
        \hline
         \multirow{2}{*}{Pre-training} & \multicolumn{2}{c|}{Market}  & \multicolumn{2}{c}{MSMT17}\\
         & mAP & R1 & mAP & R1\\
        \hline
        MocoV2+R50 \cite{fu2021unsupervised}&88.2 &94.8&53.3 &76.0 \\
        +CFS (50\%) &89.4 &95.5 &56.8 &78.8 \\
        \hline
        DINO+ViT-S  &90.3 &95.4 &64.2 &83.4 \\
        +CFS (50\%)  &91.0 &96.0 &66.1 &84.6 \\
        \hline
        \end{tabular}
    }
    \end{center}
    \vspace{-1em}
    \caption{\label{tab:cfs2} Conditional pre-training with different SSL pre-training models. We sample 50\% pre-training data based on CFS.}
\end{minipage}
\end{table*}

\renewcommand{\multirowsetup}{\centering}
\begin{table*}[tb]\small
    \begin{center}
    \begin{tabular}{cc||cc|cc||cc|cc||cc|cc}
    \hline
    & & \multicolumn{4}{c||}{Supervised ReID}  & \multicolumn{4}{c||}{USL ReID} & \multicolumn{4}{c}{UDA ReID}\\
    \cline{3-14}
     \multicolumn{2}{c||}{Pre-training}& \multicolumn{2}{c|}{Market}  & \multicolumn{2}{c||}{MSMT17} & \multicolumn{2}{c|}{Market}  & \multicolumn{2}{c||}{MSMT17} & \multicolumn{2}{c|}{MS2MA} & \multicolumn{2}{c}{MA2MS} \\
    Data & Model & mAP & R1 & mAP & R1 & mAP & R1 & mAP & R1 & mAP & R1 & mAP & R1 \\
    \hline
    \multirow{4}{*}{\shortstack{Full\\(100\%)}}&ViT-S/16
                & 90.3 & 95.4 & 64.2 & 83.4
                & 87.8 & 94.4 & 38.4 & 63.8
                & 88.5 & 95.0 & 43.9 & 67.7\\
    & ViT$_C$-S/16
                & 90.7 & 95.7 & 65.2 & 84.5
                & 88.3 & 94.6 & 39.7 & 65.2
                & 89.1 & 95.3 & 49.0 & 73.4\\
    & \multirow{2}{*}{ViT$_I$-S/16}
                & 91.1 & 95.9 & 66.8 & 85.5
                & 89.3 & 94.8 & 48.8 & 74.4
                & 89.6 & 95.6 & 55.0 & 77.9\\
    &           & {\color{gray}+0.8} &{\color{gray}+0.5} & {\color{gray}+2.6} & {\color{gray}+2.1} 
                & {\color{gray}+1.5} & {\color{gray}+0.4} & {\color{gray}+10.4} & {\color{gray}+10.6} 
                & {\color{gray}+1.1} & {\color{gray}+0.6} & {\color{gray}+11.1} & {\color{gray}+10.2}\\
    \hline
    \multirow{4}{*}{\shortstack{CFS\\(50\%)}}&ViT-S/16
                & 91.0 & 96.0 & 66.1 & 84.6 
                & 88.2 & 94.2 & 40.9 & 66.4
                & 89.4 & 95.4 & 47.4 & 70.8\\
    & ViT$_C$-S/16
                & 91.2 & 95.8 &67.8 & 85.7
                & 89.3 & 95.0 & 42.5 & 67.6
                & 89.7 & 95.5 & 55.7 & 75.5\\
    & \multirow{2}{*}{ViT$_I$-S/16}
                & 91.3 & 96.2 &68.1 & 86.1
                & 89.6 & 95.3 & 50.6 & 75.0
                & 89.9 & 95.5 & 57.8 & 79.5\\
    &           & {\color{gray}+0.3} &{\color{gray}+0.2} & {\color{gray}+2.0} & {\color{gray}+2.5} 
                & {\color{gray}+1.4} & {\color{gray}+0.9} & {\color{gray}+9.7} & {\color{gray}+8.6} 
                & {\color{gray}+0.5} & {\color{gray}+0.1} & {\color{gray}+10.4} & {\color{gray}+8.7}\\
    \hline
    \end{tabular}
    \end{center}
    \vspace{-1em}
    \caption{\label{tab:ibn}Comparison of patchify stem (ViT-S/16), convolution stem (ViT$_C$-S/16) and the proposed ICS (ViT$_I$-S/16). Gray numbers present performance improvements from the proposed ICS.}
    \vspace{-1em}
\end{table*}

\noindent\textbf{Effect of partial data.} As shown in Figure \ref{fig:cfs}, pre-training with 30\% data achieves comparable performance with the full-data pre-training. It is surprising that 40\%$\sim$60\% pre-training data even improves the performance slightly(\eg 66.3\% vs 64.2\% mAP on MSMT17). Therefore, filtering irrelevant data with CFS from the pre-training data is actually beneficial instead of harmful to the down-stream performance. For a better trade-off between accuracy and pre-training cost,  50\% of the pre-training data will be sampled in our following experiments.

\noindent\textbf{Time consumption of pre-training.} Although our conditional pre-training is two-stage, it is still more efficient. Standard pre-training of a ViT-S on full LUPerson takes 107 hours with 8$\times$V100 GPUs. In our conditional pre-training, data selection step takes 21 hours because the proxy models only need to be pre-trained for 20 epochs. For the second stage, pre-training the model on 50\% of LUPerson only takes 51 hours. In total we can still save about 30\% of pre-training time with a slight performance improvement.

\noindent\textbf{Selection strategies.} Different selection strategies are compared in Table~\ref{tab:cfs}. Randomly sampling 50\% data for pre-training does not help with the downnstream performance because it cannot reduce the noise in the original data. Although Cluster~\cite{chakraborty2020efficient} is a state-of-the-art data selection method proposed for close-set tasks, it is not suitable for the open-set person ReID problem. Our conditional pre-training based on CFS performs  the best on all benchmarks, even outperforming the full-data pre-training version, which shows its effectiveness. Some image examples selected by our method are presented in Appendix.

\noindent\textbf{Different Pre-training models.} We evaluate the effectiveness of our conditional pre-training with both ResNet50 and ViT-S/16 backbone in Table~\ref{tab:cfs2}. Consistent improvements can be observed for these two different SSL paradigms, which demonstrates the universality of our method. Table~\ref{tab:ibn} also provides more results of different ViT-S/16 models. 

\subsection{Effectiveness of ICS}

To evaluate the effectiveness of the proposed ICS, we perform a fair comparison between patchify stem, convolution stem and ICS with the ViT-S/16 backbone for both full-data pre-training and the conditional pre-retraining in Table~\ref{tab:ibn}. Three settings (\eg supervised, USL and UDA ReID) are all included for a comprehensive evaluation.

We can observe that ViT$_I$-S/16 outperforms ViT-S/16 and ViT$_C$-S/16 in most cases. For instance, for the full-data pre-training, ViT$_I$-S/16 improves the mAP by 10.4\% and 9.1\% in USL ReID on MSMT compared with ViT-S$_P$/16 and ViT$_C$-S/16, respectively. For the conditional pre-training, ICS can also consistently improve peak performance under supervised, UDA and USL settings. Therefore, ICS is proved to be an effective module for the ReID task. An interesting phenomenon is that the performance gain of ICS decreases in the conditional pre-training case, because the domain bias between pre-training and down-stream datasets has been mitigated.

\noindent{\textbf{Model Complexity and Computational Costs.}} We keep same number of transformer blocks in Table \ref{tab:ibn} to clearly compare different patch embeddings. The difference in computational complexity between ICS and convolution stem is negligible. \cite{conv_stem} has shown that the convolution stem has approximately the same complexity as a single transformer block, and adding the convolution stem while removing one transformer block can maintain similar model complexity without affecting accuracy. To further verify this point, extra experiments are conducted by removing one transformer block in the CFS-based supervised setting ( ViT$_I$-S/16 in ``CFS'' section of Table~\ref{tab:ibn}), we achieve 91.2\% mAP and 96.1\% rank-1 accuracy on Market, almost the same as the numbers in Table~\ref{tab:ibn}. Therefore, it is reasonable to expect that the performance of our other experiments will be unchanged if we remove a transformer block to maintain the complexity comparable to the vanilla ViT models.

\noindent{\textbf{Remark.}} We have also tried the  trick proposed in MocoV3~\cite{chen2021mocov3} to freeze patchify stem. It can bring in slight performance increase but is still worse than convolution stem and ICS. More experiments are included in the Appendix to show that ICS  learns more invariant feature representations.

\subsection{Comparison to State-of-the-Art methods}

\begin{table}[tb]\small
    \begin{center}
    \begin{tabular}{c@{\hspace{1pt}}c|cc|cc}
    \hline
     &  & \multicolumn{2}{c|}{Market}  & \multicolumn{2}{c}{MSMT17}\\
    Methods &Backbone & mAP & R1 & mAP & R1\\
    \hline
    BOT \cite{luo2019strong} & R50-IBN & 88.2 & 95.0 & 57.8 & 80.7 \\
    BOT$^*$ \cite{luo2019strong} & R50-IBN & 90.1 & 95.6 & 60.8 & 81.4  \\
    MGN \cite{MGN} & R50$\uparrow$384 &87.5 &95.1 &63.7 &85.1 \\
    SCSN \cite{SCSN} & R50$\uparrow$384 &88.5 &95.7 &58.5 &83.8 \\
    ABDNet \cite{ABD-Net} & R50$\uparrow$384 &88.3 &95.6 &60.8 &82.3\\
    TransReID\cite{he2021transreid} & ViT-B$\uparrow$384 &89.5 &95.2 &\underline{69.4} &\underline{89.2} \\
    MoCoV2$^*$ \cite{fu2021unsupervised} & MGN$\uparrow$384 & \underline{91.0} & \underline{96.4} & 65.7 & 85.5 \\
    \hline
    Ours$^*$ & ViT$_I$-S&91.3 &96.2 &68.1&86.1 \\
    Ours$^*$ & ViT$_I$-S$\uparrow$384 &91.7 &96.3 &70.5 &87.8 \\
    Ours$^*$ & ViT$_I$-B$\uparrow$384 &\textbf{93.2 }&\textbf{96.7} &\textbf{75.0} &\textbf{89.5}\\
    \hline
    \end{tabular}
    \end{center}
    \vspace{-1em}
    \caption{\label{tab:sota1}Comparison to state-of-the-art methods in supervised ReID. MGN is the improved version in fast-reid. $^*$ means that backbones are pre-trained on LUPerson. $\uparrow$384 represents that images are resized to $384 \times 128$. From the perspective of computational complexity, ViT$_I$-S and ViT$_I$-B can be compared with R50/R50-IBN and MGN, respectively. R50-IBN stands for ResNet50-IBN-a.}
\end{table}

\textbf{Supervised ReID.} We compare to some of the state-of-the-art methods on supervised ReID in Table~\ref{tab:sota1}. Our ViT-S$\uparrow$384 outperforms MGN and ABDNet on both benchmarks by a large margin. It is worth noting that TransReID is pre-trained on \textbf{\textit{ ImageNet-21K}} and integrates additional camera information and part features to achieve the current performance. With self-supervised pre-training on LUPerson, our ViT-B$\uparrow$384 obtains 93.2\%/75.0\% mAP and 96.7\%/89.5\% rank-1 accuracy on Market1501/MSMT17 datasets, significantly outperforming TransReID with no additional modules. It is also  observed that MGN benefits from the self-supervised pre-training on LUPerson via MocoV2, but is still inferior to our results.

\begin{table}[tb]\small
    \begin{center}
    \begin{tabular}{cc|cc|cc}
    \hline
     &  & \multicolumn{2}{c|}{Market} & \multicolumn{2}{c}{MSMT17}\\
    Methods &Backbone & mAP & R1 & mAP & R1\\
    \hline
    MMCL \cite{MMCL} &R50 & 45.5 & 80.3 &11.2 & 35.4\\
    HCT \cite{HCT} & R50 & 56.4 & 80.0 & - & -\\ 
    IICS \cite{IICS} &R50 & 72.9 & 89.5 &26.9&52.4\\
    SPCL$^*$\cite{SPCL} & R50 &76.2 &90.2 &-&-\\
    C-Contrast \cite{dai2021cluster} & R50 &82.6 &93.0 &33.1 &63.3 \\
    C-Contrast$^*$ \cite{dai2021cluster} & R50 &84.0 &93.4 &31.4 &58.8 \\
    C-Contrast$^*$ \cite{dai2021cluster} & R50-IBN & \underline{86.4} & \underline{94.2} & \underline{39.8} & \underline{66.1}  \\
    \hline
    Ours$^*$ & ViT-S&88.2 &94.2 &40.9 &66.4 \\
    Ours$^*$ & ViT$_I$-S&\textbf{89.6} &\textbf{95.3} &\textbf{50.6}&\textbf{75.0 }\\
    \hline
    \end{tabular}\end{center}
    \vspace{-1em}
    \caption{\label{tab:sota2}Comparison to state-of-the-art methods in USL ReID. $^*$ means that backbones are pre-trained on LUPerson.}
\end{table}

\begin{table}[tb]\small
    \begin{center}
    \begin{tabular}{cc|cc|cc}
    \hline
    &  &\multicolumn{2}{c|}{MS2MA}  &\multicolumn{2}{c}{MA2MS}\\
    Methods &Backbone & mAP & R1 & mAP & R1\\
    \hline
    DG-Net++ \cite{DG-Net++}& R50 &64.6 & 83.1 & 22.1 &48.4\\
    MMT\cite{MMT} & R50 &75.6 & 89.3 & 24.0 & 50.1 \\
    SPCL\cite{SPCL} & R50 &77.5 &89.7 &26.8 &53.7 \\
    SPCL\cite{SPCL} & R50-IBN &79.9 &92.0 &31.0 &58.1 \\
    C-Contrast \cite{dai2021cluster} & R50 &82.4 &92.5 &33.4 &60.5 \\
    C-Contrast$^*$ \cite{dai2021cluster} & R50 &85.1 &94.4 &28.3 &53.8 \\
    C-Contrast$^*$ \cite{dai2021cluster} & R50-IBN  & \underline{86.9} & \underline{94.6} & \underline{42.6} & \underline{69.1}\\
    \hline
    Ours$^*$ & ViT-S    &89.4 & 95.4 & 47.4 & 70.8\\
    Ours$^*$ & ViT$_I$-S&\textbf{89.9} &\textbf{95.5} &\textbf{57.8} &\textbf{79.5}\\
    \hline
    \end{tabular}\end{center}
    \vspace{-1em}
    \caption{\label{tab:sota3}Comparison to state-of-the-art methods in UDA ReID. $^*$ means that backbones are pre-trained on LUPerson.}
\end{table}

\textbf{USL ReID.} Our methods are compared to MMCL \cite{MMCL}, HCT \cite{HCT}, IICS \cite{IICS}, SPCL\cite{SPCL} and C-Contrast \cite{dai2021cluster} in Table~\ref{tab:sota2}, where the last two methods also adopt the pre-trained model on LUPerson. Our best results boost the mAP performance by 3.2\% (89.6\% vs 86.4\%) and 10.8\% (50.6\% vs 39.8\%) on Market and MSMT17, respectively.

\textbf{UDA ReID.} Some latest UDA-ReID methods are compared in the Table~\ref{tab:sota3}. Among all existing state-of-the-art methods, C-Contrast achieves the best performance at 86.9\% and 42.6\% mAP on MS2MA and MA2MS, respectively. Our method with both ViT-S and ViT$_I$-S surpasses C-Contrast by a large margin. Especially, ViT$_I$-S obtains 89.9\% mAP on MS2MA and 57.8\% mAP on MA2MS, which are already comparable to many supervised methods.

\section{Conclusions and Discussions}

This paper mitigate the domain gap between pre-training and ReID datasets for transformer-based person ReID from aspects of data and model. After observing that the existing pre-training paradigms of person ReID cannot perform well for transformer-based backbones, we investigate SSL methods with ViTs on the LUPerson and find DINO is the most suitable pre-training method for transformer-based ReID. To further bridge the gap between pre-training and ReID datasets, we propose a conditional pre-training method based on CFS to select relevant pre-training data to the target domain. The proposed method can speed up pre-training without performance drop. For the model structure, a ReID-specific patch embedding called IBN-based convolution stem is proposed to improve the peak performance. We believe the promising performance will inspire more work to study the SSL pre-training for the transformer-based models towards person ReID, \ie, a more suitable ViT variant, or a ReID-specific pre-training framework, etc. 

\textbf{Limitations.} The pre-trained models are only suitable for the person ReID but cannot perform well on other less related tasks such as vehicle ReID, human parsing, or image classification, etc.

{\small
\bibliographystyle{ieee_fullname}
\bibliography{egbib}
}

\clearpage
\appendix
\section{Conditional Pre-training on ImageNet}

In this section,  our proposed conditional pre-training is further applied on ImageNet in Table~\ref{tab:apx1}. Since most of the images in ImageNet are irrelevant to person images, it is difficult to select enough data for the ReID task. Therefore, the pre-training on 50\% data is inferior to the full pre-training version, but pre-training on CFS-selected data still outperforms the random sampling strategy, which shows the effectiveness of the CFS-based selection strategy.

\begin{table}[htb]
    \begin{center}
    \begin{tabular}{c|cc|cc}
        \hline
        \multirow{2}{*}{Pre-training} & \multicolumn{2}{c|}{Market}  & \multicolumn{2}{c}{MSMT17}\\
         & mAP & R1  & mAP & R1\\
        \hline
        Full (100\%)    &84.6 &93.1 &54.8 &76.7 \\
        Random (50\%)   &81.6 &91.4 &52.5 &75.2 \\
        CFS (50\%)      &83.8 &92.8 &53.8 &76.2 \\
        \hline
    \end{tabular}
    \end{center}
    \caption{\label{tab:apx1} Conditional Pre-training on ImageNet. The backbone is vanilla ViT-S. }
\end{table}

\section{Image Examples with High CFS}
\begin{figure}[htb]
    \centering
	\includegraphics[width=0.45\textwidth]{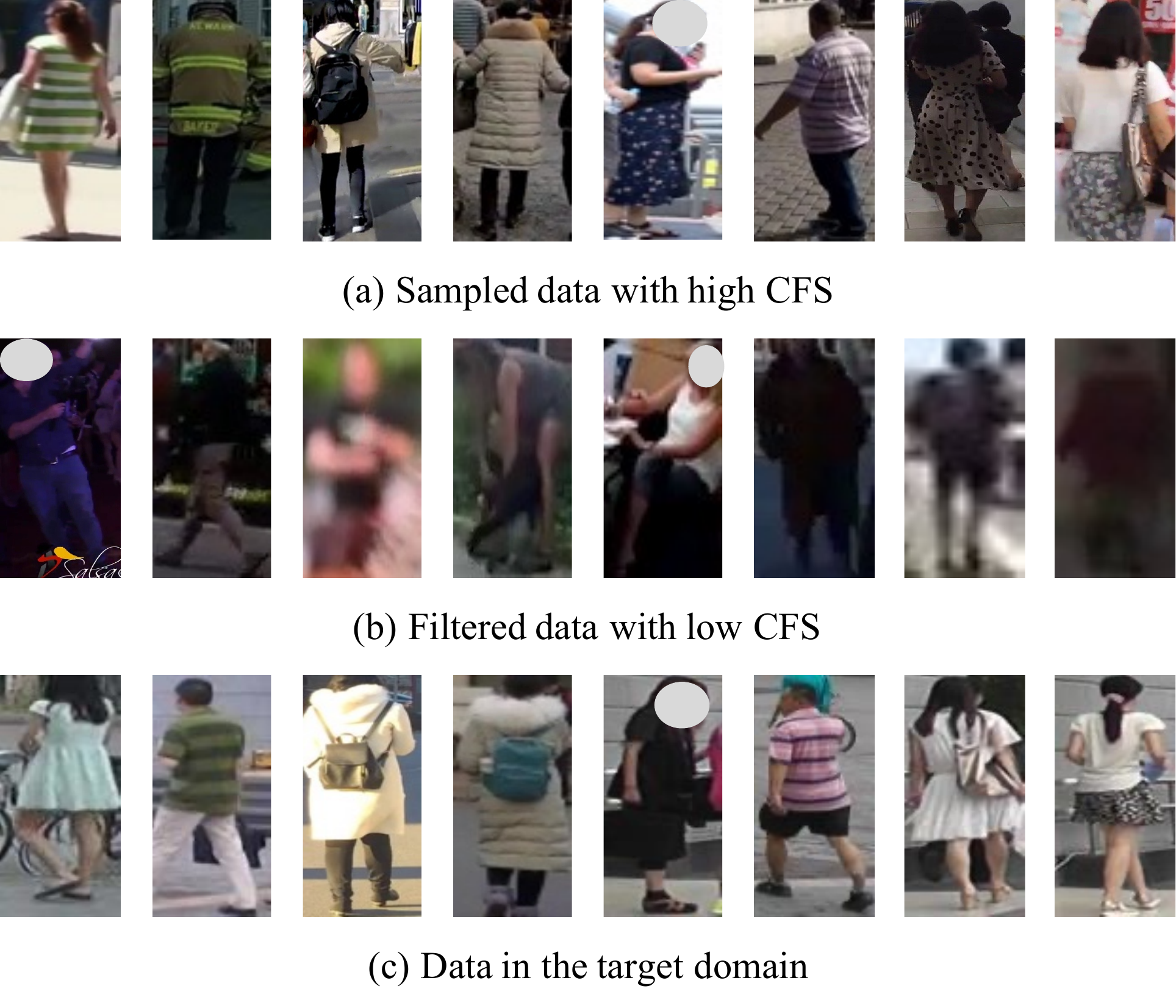}
	\caption{Some random examples of selected data with high CFS and filtered data with low CFS. Images in the first two rows are from the source domain (LUPerson). The last row shows images sampled from the target domain (Market and MSMT17). Images in each row have been rearranged for the best view.} 
	\label{fig:samples}
	\vspace{-1em}
\end{figure}

As Figure~\ref{fig:samples} shows, the sampled images with higher CFS are more similar to the data in the target domain. In contrast, the filtered images with lower CFS are of low-quality or have great domain gap with the target datasets. These filtered images cannot provide enough discriminative ID information during pre-training. Therefore, our conditional pre-training can effectively mitigate the domain gap between the pre-training and target datasets by removing those irrelevant or harmful data.

\section{Analysis of Feature Representation}

Table~\ref{tab:feat} shows an experiment on Market-1501 to analyze the feature invariance of pre-trained models. The original dataset is extended into six variations by applying six different augmentations to each image, to simulate some important appearance variances in the ReID task (examples shown in Figure \ref{fig:example}). Then, Centered Kernel Alignment (CKA)~\cite{cka} scores can evaluate the similarity between the features of the origin dataset and the features of the each simulated dataset. The higher CKA score is, the features provided by the pre-trained model show better invariance on this type of augmentation. ViT$_I$-S achieves the best CKA scores for all types of appearance variance, showing that ICS is beneficial for learning invariant features.

\begin{figure}[htb]
    \centering
	\includegraphics[width=0.45\textwidth]{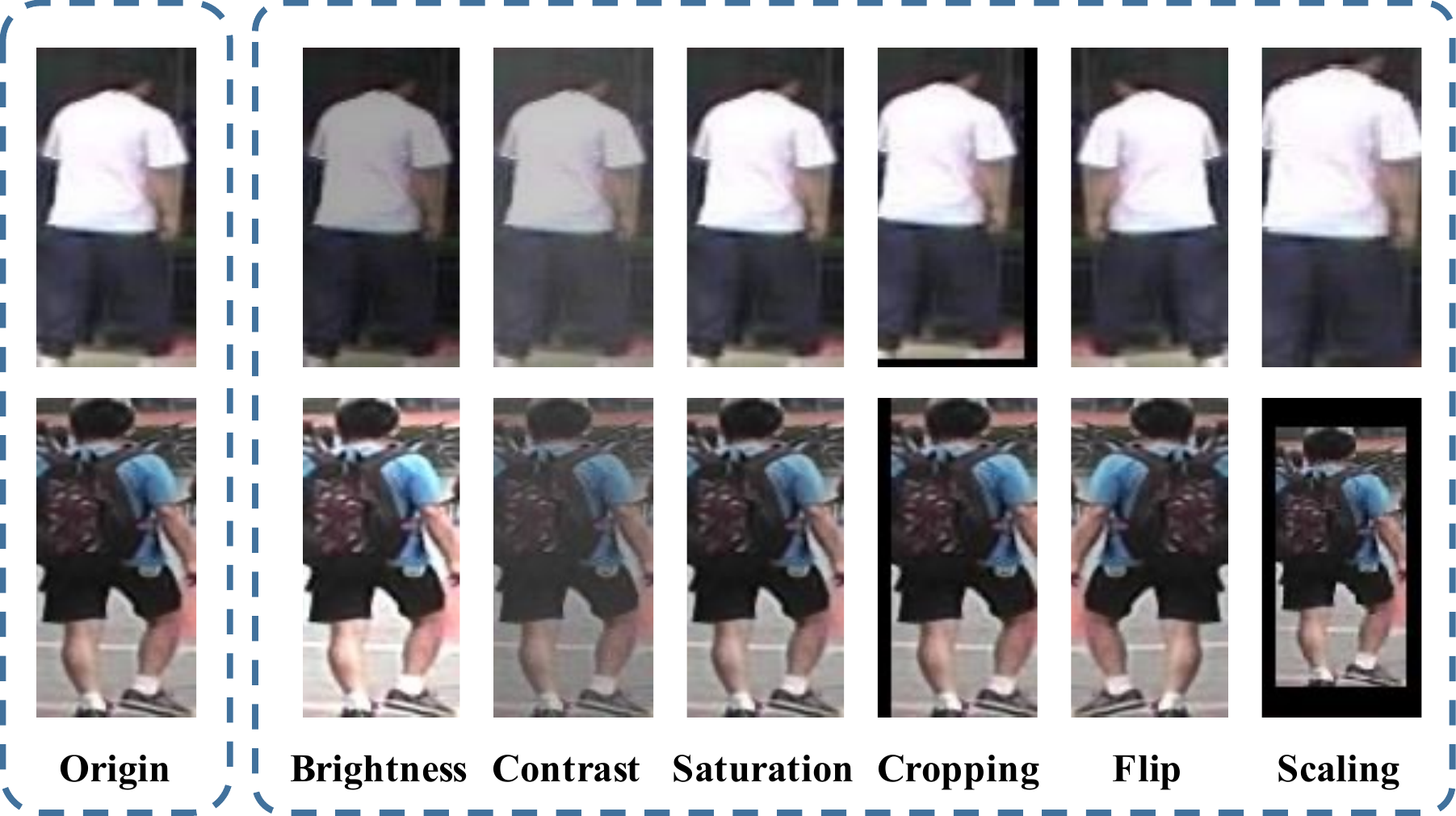}
	\caption{Examples of generated images. Six different augmentations are included, \ie brightness, contrast, saturation, cropping, flip, and scaling, which are all important for appearance variance in ReID. The first two types of augmentation are related to color information and the last four are related to texture information.} 
	\label{fig:example}
	\vspace{-1em}
\end{figure}

\begin{table}[htb]\small
    \begin{center}\setlength\tabcolsep{5pt}
    \begin{tabular}{c|cccccc}
    \hline
    Pre-train & Brig & Cont & Sat & Crop & Flip & Scale \\
    \hline
    ViT$_I$-S &\textbf{0.809} &\textbf{0.818} &\textbf{0.998} &\textbf{0.927} &\textbf{0.996} &\textbf{0.820} \\
    ViT$_C$-S & 0.727 & 0.691 & 0.997& 0.919 & 0.996 & 0.794 \\
    ViT-S   &0.702 &0.664 & 0.997 &0.877 &0.996 & 0.728\\
    \hline
    \end{tabular}\end{center}
    \vspace{-1em}
    \caption{\label{tab:feat} Feature invariance of different patch embedding methods. 
     The table compares the CKA scores of different pre-trained models for  different augmentations. Brig, Cont, Sat, Crop, Flip, and Scale represent augmentation of brightness, contrast, saturation, cropping, flip, and scaling, respectively.} 
     \vspace{-1.5em}
\end{table}

\section{Improvements of LUPerson Pre-training}
\vspace{-0.5em}
\renewcommand{\multirowsetup}{\centering}
\begin{table}[htb]\small
    \begin{center}
    \setlength\tabcolsep{5pt}
    \begin{tabular}{ccc|cc|cc}
    \hline
     \multicolumn{3}{c|}{Pre-training}  & \multicolumn{2}{c|}{Market}  & \multicolumn{2}{c}{MSMT17}\\
    Models &Data & $\#$Images & mAP & R1 & mAP & R1\\
    \hline
    \multirow{2}{*}{ViT-S}& IMG-1K &1.28M &85.0 &93.8 &53.5 &75.2\\
                          &LUP & 4.18M &90.3 &95.4 &64.2 &83.4\\
    \hline
    \multirow{2}{*}{ViT-B}& IMG-21K &14.2M &86.8 &94.7 &61.0 &81.8\\
                          &LUP & 4.18M &92.5 &96.5 &70.1 &87.3\\
    \hline
    \end{tabular}
    \end{center}
    \vspace{-1em}
    \caption{\label{tab:ssl2} SSL Pre-training on LUPerson can consistently improve the performance of supervised pre-training on ImageNet-1k or ImageNet-21k. Conditional pre-training or ICS is not included.}
    \vspace{-2em}
\end{table}

\onecolumn
\begin{algorithm*}
\caption{Source codes of Patch Embedding for \textit{IBN-based Convolution Stem}.}\label{ics}
\begin{lstlisting}[language=Python]
class IBN(nn.Module):
    def __init__(self, planes):
        super(IBN, self).__init__()
        half1 = int(planes/2)
        self.half = half1
        half2 = planes - half1
        self.IN = nn.InstanceNorm2d(half1, affine=True)
        self.BN = nn.BatchNorm2d(half2)

    def forward(self, x):
        split = torch.split(x, self.half, 1)
        out1 = self.IN(split[0].contiguous())
        out2 = self.BN(split[1].contiguous())
        out = torch.cat((out1, out2), 1)
        return out

class PatchEmbed(nn.Module):
    def __init__(self, img_size=224, patch_size=16, stride_size=16, in_chans=3, embed_dim=768, stem_conv=False):
        super().__init__()
        img_size = to_2tuple(img_size)
        patch_size = to_2tuple(patch_size)
        stride_size_tuple = to_2tuple(stride_size)
        self.num_x = (img_size[1] - patch_size[1]) // stride_size_tuple[1] + 1
        self.num_y = (img_size[0] - patch_size[0]) // stride_size_tuple[0] + 1
        print('using stride: {}, and patch number is num_y{} * num_x{}'.format(stride_size, self.num_y, self.num_x))
        self.num_patches = self.num_x * self.num_y
        self.img_size = img_size
        self.patch_size = patch_size

        self.stem_conv = stem_conv
        if self.stem_conv:
            hidden_dim = 64
            stem_stride = 2
            stride_size = patch_size = patch_size[0] // stem_stride
            self.conv = nn.Sequential(
                nn.Conv2d(in_chans, hidden_dim, kernel_size=7, stride=stem_stride, padding=3,bias=False),
                IBN(hidden_dim),
                nn.ReLU(inplace=True),
                nn.Conv2d(hidden_dim, hidden_dim, kernel_size=3, stride=1,padding=1,bias=False),
                IBN(hidden_dim),
                nn.ReLU(inplace=True),
                nn.Conv2d(hidden_dim, hidden_dim, kernel_size=3, stride=1,padding=1,bias=False),
                nn.BatchNorm2d(hidden_dim),
                nn.ReLU(inplace=True),
            )
            in_chans = hidden_dim
        self.proj = nn.Conv2d(in_chans, embed_dim, kernel_size=patch_size, stride=stride_size)

    def forward(self, x):
        if self.stem_conv:
            x = self.conv(x)
        x = self.proj(x)
        x = x.flatten(2).transpose(1, 2)
        return x

#   ICS for ViT-S in this paper
#   self.patch_embed = PatchEmbed(img_size=(256,128), patch_size=16, stride_size=16, in_chans=3 ,embed_dim=384, stem_conv=True)
\end{lstlisting}
\end{algorithm*}

\section{Theoretical Analysis of Generalization Performance} 
In this section, we give a theoretical analysis of generalization performance relating source and target data. In particular, we consider excess risk bound (ERB) as a measure of generalization performance, which is commonly used in the learning theory literature~\cite{vapnik2013nature}. The following informal theorem highlights the factors that influence the generalization of target task, which dramatically simplify Theorem~\ref{thm:main}.
\begin{thm}[Informal Version of Theorem~\ref{thm:main}]\label{thm:main:informal} 
Under the condition that the sample size of source data is larger than that of target data, then we have the following tow propositions to help improve generalization performance of target task:
\begin{itemize}
    \item use the examples from the source data that close to target data;
    \item use a small learning rate to slightly update the backbone during the fine-tuning process.
\end{itemize}
\end{thm}
In the remaining of this section, we first introduce the problem setting, and then present the formal theoretical result with its proof. To simplify the presentation of the proof, we include some supporting lemmas at the end of this section.
\subsection{Problem Setting}
To give an insight of the proposed method from the view of theoretical side, we formalize the problem as follows. We consider binary classification for simplicity. A domain is defined as a pair consisting of a distribution $\D$ on inputs $\X$ and a labeling function $f: \X \rightarrow \Y$ with $\Y := \{0, 1\}$. The labeling function can have a fractional (expected) value when labeling occurs non-deterministically. We denote by $\langle\D_s, f_s\rangle$ the source domain and $\langle\D_t, f_t\rangle$ the target domain. A hypothesis is a function $h\in\mathcal H: \X \rightarrow \Y$, where $\mathcal H$ is a hypothesis space on $\X$ with VC dimension~\cite{vapnik1971uniform} $d$. The loss on source domain that a hypothesis $h$ disagrees with a labeling function $f$ is defined as
\begin{align}
     \F_s(h,f) := \E_{\x\sim\D_s}\left[|h(\x)-f(\x)| \right],
\end{align}
where $f$ can also be a hypothesis. Its empirical version is written as
\begin{align}
     \widehat\F_s(h) := \frac{1}{|\mathcal U_s|}\sum_{\x\in\mathcal U_s}|h(\x)-f(\x)|,
\end{align}
where $\mathcal U_s$ is sampled from $\D_s$ and $|\mathcal U_s|$ is the sample size.
Similarly, for target domain we define
\begin{align}
     \F_t(h,f) := \E_{\x\sim\D_t}\left[|h(\x)-f(\x)| \right],\\
     \widehat\F_t(h) := \frac{1}{|\mathcal U_t|}\sum_{\x\in\mathcal U_t}|h(\x)-f(\x)|,
\end{align}
where $\mathcal U_t$ is sampled from $\D_t$ and $|\mathcal U_t|$ is the sample size..
Since the label function is deterministic for target task, the risk of a hypothesis is defined as
\begin{align}
     \F_t(h) := \E_{\x\sim\D_t}\left[|h(\x)-f_t(\x)| \right],
\end{align}
and the empirical risk is 
\begin{align}
     \widehat\F_t(h) := \frac{1}{|\mathcal U_t|}\sum_{\x\in\mathcal U_t}|h(\x)-f_t(\x)|.
\end{align}
Let define 
\begin{align*}
    h_* := &\arg\min_{h\in\mathcal H}\F_t(h),\\
    \widetilde h_*,\widetilde f_* := &\arg\min_{h\in\mathcal H, f\in\mathcal H}\F_s(h,f),\\
    \F_* := & \F_t(h_*) + \F_s(\widetilde h_*,\widetilde f_*).
\end{align*}
The excess risk is defined by
\begin{align}
    \F_t(h) - \F_t(h_*),
\end{align}
which can be considered as a measure of generalization performance. Our goal is to minimize the excess risk.

For a hypothesis space $\mathcal H$, we give the definitions of the symmetric difference hypothesis space $\mathcal H\Delta\mathcal H$ and its divergence following by~\cite{ben2010theory}.
\begin{definition}\cite{ben2010theory}\label{def:sym:hyp}
For a hypothesis space $H$, the symmetric difference hypothesis space $\mathcal H\Delta\mathcal H$ is the set of hypotheses
$g\in \mathcal H\Delta\mathcal H  \Leftrightarrow g(\x)=h(\x)\oplus h'(\x)$ forsome $h,h'\in\mathcal H$,
where $\oplus$ is the XOR function. That is, every hypothesis $g\in \mathcal H\Delta\mathcal H$  is the set of disagreements between two hypotheses in $\mathcal H$.
\end{definition}
\begin{definition}\cite{ben2010theory}\label{def:sym:hyp:divergence}
 For a hypothesis space $\mathcal H$, the $\mathcal H\Delta\mathcal H$-distance between two distributions $\D_1$ and $\D_2$ is defined as
\begin{align*}
    d_{\mathcal H \Delta \mathcal H}(\D_1,\D_2)  = \sup_{h,h'\in\mathcal H}\left|\text{Pr}_{\x\sim\D_1}(h(\x)\neq h'(\x)) - \text{Pr}_{\x\sim\D_2}(h(\x)\neq h'(\x)) \right|.
\end{align*}
\end{definition}
{\bf Remark.} The empirical version of $\mathcal H\Delta\mathcal H$-distance, denoted by $\widehat d_{\mathcal H \Delta \mathcal H}(\mathcal U_s,\mathcal U_t)$, can be considered as the a measure of similarity between the sampled source data $\mathcal U_s$ and the sampled target data $\mathcal U_t$.
\subsection{Main Result and its Proof}
\begin{thm}[Formal Version]\label{thm:main}
Let $\mathcal U_s$, $\mathcal U_t$ finite samples of sizes $m$, $n$ ($m\gg n$) drawn i.i.d. according $\D_s$, $\D_t$ respectively. $\widehat d_{\mathcal H \Delta \mathcal H}(\mathcal U_s,\mathcal U_t)$ is the empirical $\mathcal H \Delta \mathcal H$-divergence between samples, then for any $\delta\in(0,1)$, with probability at least $1-\delta$, we have
\begin{align}\label{bound:thm:main}
    \nonumber\F_t(h) - \F_t(h_*)\le &     \frac{3}{2}\widehat d_{\mathcal H \Delta \mathcal H}(\mathcal U_s,\mathcal U_t)  + \widehat\F_t(h,\widetilde h_*)  + \F_t(h_*,\widetilde f_*) + \F_s(\widetilde h_*,\widetilde f_*)  \\
    &+ \sqrt{\frac{\log(8/\delta)}{2n}}+12\sqrt{\frac{2d\log(2n)+\log(8/\delta)}{n}}.
\end{align}
\end{thm}
{\bf Remark.} We make several comments of the ERB in Theorem~\ref{thm:main}. First, the bound shows that the unlabeled empirical $\mathcal H\Delta\mathcal H$-divergence (i.e., $\widehat d_{\mathcal H \Delta \mathcal H}(\mathcal U_s,\mathcal U_t)$, the first term of the left hand side of bound (\ref{bound:thm:main})) is important to improve the generalization performance (reduce the excess risk), indicating that when the used source data and target data are close, it will have a good generalization performance in target task. Second, the bound needs the term $\widehat\F_t(h,\widetilde h_*)$ is small for a hypothesis $h$ and the optimal hypothesis $\widetilde h_*$ of source domain, which means that $h$ should be not too far from $\widetilde h_*$. This is consistent with the practice that usually a small learning rate is used to slightly update the backbone during the fine-tuning process. Finally, the term $\F_s(\widetilde h_*,\widetilde f_*)$ in the bound tells us that the source error obtained during the pre-training process is also important to the target performance. 
\begin{proof}[Proof of Theorem~\ref{thm:main}]
By Lemma~\ref{lem:triangle} we have
    \begin{align*}
    \F_t(h) - \F_t(h_*) \le &  \F_t(h,h_*) \\
    \le & \F_s(h,h_*) + \left|\F_t(h,h_*) - \F_s(h,h_*)\right| \\
    \le & \F_s(h,h_*) + \sup_{h,h'\in\mathcal H}\left|\F_t(h,h') - \F_s(h,h')\right| \\
    = & \F_s(h,h_*) + \sup_{h,h'\in\mathcal H}\left|\text{Pr}_{\x\sim\D_t}(h(\x)\neq h'(\x)) - \text{Pr}_{\x\sim\D_s}(h(\x)\neq h'(\x)) \right| \\
    = &  \F_s(h,h_*) + \frac{1}{2}d_{\mathcal H \Delta \mathcal H}(\D_s,\D_t) \quad\text{ (by Definition~\ref{def:sym:hyp})}\\
    \le &  \F_s(h_*,\widetilde f_*)  + \F_s(h,\widetilde h_*)  + \F_s(\widetilde h_*,\widetilde f_*) +  \frac{1}{2}d_{\mathcal H \Delta \mathcal H}(\D_s,\D_t) \\
    \overset{(a)}{\le} &  \F_s(h_*,\widetilde f_*)  + \F_s(h,\widetilde h_*)  + \F_s(\widetilde h_*,\widetilde f_*) +  \frac{1}{2}\widehat d_{\mathcal H \Delta \mathcal H}(\mathcal U_s,\mathcal U_t) + 4\sqrt{\frac{2d\log(2n)+\log(2/\delta)}{n}}\\
    \overset{(b)}{\le} &  \F_s(h_*,\widetilde f_*)  + \widehat\F_s(h,\widetilde h_*)  + \sqrt{\frac{\log(4/\delta)}{2m}} + \F_s(\widetilde h_*,\widetilde f_*) \\
    &+  \frac{1}{2}\widehat d_{\mathcal H \Delta \mathcal H}(\mathcal U_s,\mathcal U_t) + 4\sqrt{\frac{2d\log(2n)+\log(4/\delta)}{n}}.
\end{align*}
Here (a) uses Lemma~\ref{lem:concentration} with the fact that since every $g\in \mathcal H\Delta\mathcal H$ can be represented as a linear threshold network of depth 2 with 2 hidden units, the VC dimension of $ \mathcal H\Delta\mathcal H$ is at most $2d$~\cite{anthony1999neural,ben2010theory}; (b) uses Hoeffding's inequality in Lemma~\ref{lem:hoeffding}. Similarly, we have
\begin{align*}
    \F_t(h) - \F_t(h_*)\le & \F_t(h_*,\widetilde f_*)  + \F_t(h,\widetilde h_*)  + \F_s(\widetilde h_*,\widetilde f_*)  +   \frac{3}{2}\widehat d_{\mathcal H \Delta \mathcal H}(\mathcal U_s,\mathcal U_t) + 12\sqrt{\frac{2d\log(2n)+\log(8/\delta)}{n}}\\
    \le & \F_t(h_*,\widetilde f_*)  + \widehat\F_t(h,\widetilde h_*)  + \F_s(\widetilde h_*,\widetilde f_*)  +   \frac{3}{2}\widehat d_{\mathcal H \Delta \mathcal H}(\mathcal U_s,\mathcal U_t) \\
    &+ \sqrt{\frac{\log(8/\delta)}{2n}}+12\sqrt{\frac{2d\log(2n)+\log(8/\delta)}{n}}.
\end{align*}
\end{proof}

\subsection{Supporting Lemmas}
\begin{lemma}\label{lem:triangle}
Define $\F(h,f) := \E_{\x\sim\D}[|h(\x)-f(\x)|]$. For any $h_1,h_2,h_3\in\mathcal H$, we have $\F(h_1,h_2) = \F(h_2,h_1)$ and
\begin{align}
    \F(h_1,h_2) \le \F(h_1,h_3) + \F(h_2,h_3).
\end{align}
\end{lemma}
\begin{proof}\renewcommand{\qedsymbol}{}
 By the definition of $\F(h,f)$, we know this is symmetric. On the other hand, we have
 \begin{align*}
     \F(h_1,h_2) = & \E_{\x\sim\D}[|h_1(\x)-h_2(\x)|] \\
     \le& \E_{\x\sim\D}[|h_1(\x)-h_3(\x)|+ |h_2(\x)-h_3(\x)|]\\
     \le& \E_{\x\sim\D}[|h_1(\x)-h_3(\x)|] + \E_{\x\sim\D}[|h_2(\x)-h_3(\x)|]\\
     =& \F(h_1,h_3) + \F(h_2,h_3).
 \end{align*}
\end{proof}
\begin{lemma}\label{lem:concentration}
    Let $\mathcal H$ be a hypothesis space on $\X$ with VC dimension $d$. $\D_1$ and $\D_2$ be two distribution on $\X$ and $\mathcal U_1$, $\mathcal U_2$ finite samples of sizes $n_1$, $n_2$ ($n_1\gg n_2$) drawn i.i.d. according $\D_1$, $\D_2$ respectively. $\widehat d_{\mathcal H \Delta \mathcal H}(\mathcal U_1,\mathcal U_2)$ is the empirical H-divergence between samples, then for any $\delta\in(0,1)$, with probability at least $1-\delta$, we have
    \begin{align}\label{lem:ineq:concentration}
        d_{\mathcal H \Delta \mathcal H}(\D_1,\D_2) \le \widehat d_{\mathcal H \Delta \mathcal H}(\mathcal U_1,\mathcal U_2) + 4\sqrt{\frac{d\log(2n_2)+\log(2/\delta)}{n_2}}.
    \end{align}
\end{lemma}
\begin{proof}\renewcommand{\qedsymbol}{}
    This lemma can be proved by a slight modification of Theorem 3.4 of~\cite{kifer2004detecting}. To this end, by setting $\epsilon = 4\sqrt{\frac{d\log(2n_2)+\log(2/\delta)}{n_2 }}$ and $\widetilde\epsilon = 4\sqrt{\frac{d\log(2n_1)+\log(2/\delta)}{n_1}}$, then we know $\epsilon\le\widetilde\epsilon$ (since $n_1\gg n_2$) and
    \begin{align*}
        &(2n_1)^d \exp\left(-\frac{\epsilon^2 n_1}{16}\right) \le (2n_1)^d \exp\left(-\frac{\widetilde\epsilon^2 n_1}{16}\right) = \delta/2,\\
        &(2n_2)^d \exp\left(-\frac{\epsilon^2 n_2}{16}\right) = \delta/2.
    \end{align*}
    We then obtain (\ref{lem:ineq:concentration}) by using Theorem 3.4 of~\cite{kifer2004detecting}.
\end{proof}
\begin{lemma}[Hoeffding's inequality~\cite{hoeffding1963probability}]\label{lem:hoeffding}
Let $X_1, \dots, X_n$ be i.i.d. random variables with bounded intervals: $a_i\le X_i\le b_i$, $i = 1,\dots, n$. Let $\bar X = \frac{1}{n}\sum_{i=1}^{n}X_i$, then we have
\begin{align*}
\text{Pr}\left(\left|\bar X - \E[\bar X]\right| \ge c \right) \le 2\exp\left(-\frac{2n^2c^2}{\sum_{i=1}^{n}(b_i-a_i)^2} \right).
\end{align*}
\end{lemma}
\end{document}